\documentclass{article} 
\usepackage{iclr2027_conference,times}
\iclrfinalcopy

\usepackage{amsmath,amsfonts,bm}

\def\eqref#1{equation~\ref{#1}}

\def\1{\bm{1}}

\DeclareMathAlphabet{\mathsfit}{\encodingdefault}{\sfdefault}{m}{sl}
\SetMathAlphabet{\mathsfit}{bold}{\encodingdefault}{\sfdefault}{bx}{n}

\usepackage{natbib}
\usepackage{hyperref}
\usepackage{url}

\usepackage{booktabs}

\usepackage{amsmath}
\usepackage{amssymb}
\usepackage{mathtools}
\usepackage{amsthm}

\theoremstyle{plain}
\newtheorem{theorem}{Theorem}[section]

\newtheorem{lemma}[theorem]{Lemma}

\theoremstyle{definition}
\newtheorem{definition}[theorem]{Definition}

\theoremstyle{remark}

\title{A Generalization Theory for JEPA-Based World Models}

\author{
  {\bf Jingyi Cui}$^{1,3}$\thanks{Equal contribution} \qquad
  {\bf Qi Zhang}$^{1}$\footnotemark[1] \qquad
  {\bf Hongwei Wen}$^{2}$ \qquad
  {\bf Letian Cheng}$^{1}$ \qquad
  {\bf Yisen Wang}$^{1}$\thanks{Corresponding author: Yisen Wang (yisen.wang@pku.edu.cn)} \\
  $^1$State Key Lab of General AI, 
  School of Intelligence Science and Technology, Peking University \\
  $^2$University of Sydney \\
  $^3$JD.com
}

\begin{document}

\maketitle

\begin{abstract}
	Joint Embedding Predictive Architectures (JEPAs) have emerged as a promising paradigm for world modeling by predicting future representations rather than reconstructing observations. Despite their empirical success, the theoretical understanding of JEPA-based world models remains limited. In this paper, we develop the first generalization theory for JEPA-based world models. We formulate JEPA pretraining as a conditional spectral graph learning problem and show that the JEPA objective is equivalent to a low-rank factorization of an action-conditioned co-occurrence matrix. Building on this characterization, we establish a connection between JEPA pretraining error and downstream planning regret, leading to a finite-sample generalization bound for JEPA-based world models. Our analysis reveals an approximation--sample error trade-off governed by latent dimensionality, providing insights into latent- versus input-level prediction. The theoretical trade-off motivates us to investigate a simple Multi-Depth JEPA extension that combines representations from different network depths. In validation experiments, latent-level prediction exhibits greater robustness under unpredictable observations and difficult planning tasks. Multi-depth representations further improve planning success over single-depth representations and a dimensionality-matched control.
\end{abstract}

\section{Introduction}

Recent progress in self-supervised learning has increasingly shifted the focus of representation learning from discriminative objectives toward predictive world modeling. Among the emerging paradigms, Joint Embedding Predictive Architectures (JEPAs) \citep{lecun2022path} excel by predicting latent-level information while avoiding the inefficiency of input-level generation. This perspective has motivated a rapidly growing line of research on predictive latent world models, including V-JEPA \citep{assran2025v}, VL-JEPA \citep{chen2025vl}, LeWM \citep{maes2026leworldmodel}, etc.

Despite the growing empirical success of JEPA world models, their theoretical understanding remains limited. 
Existing theoretical studies on JEPAs have primarily focused more on understanding JEPAs as a framework for latent space learning. \citet{littwin2024jepa} analyzed that JEPAs preferentially learn high-influence semantic features while suppressing noisy or weakly informative signals. \citet{balestriero2025gaussian} reveals that JEPA anti-collapse regularization implicitly performs density estimation of the input observations. Very recently, \citet{klindt2026does} proved the identifiability of the Gaussian-regularized JEPA  world model.
Nonetheless, these works either studied JEPA in the static representation learning setting or based the theoretical analysis on parametric modeling. To the best of our knowledge, there is currently no theoretical framework explaining how JEPAs generalize as a world model framework in real-world action planning.
Since actions are planned in latent space but evaluated in downstream tasks, understanding how JEPA pretraining error affects planning performance is essential. Moreover, it remains unclear how latent dimensionality affects the trade-off between predictive capacity and finite-sample generalization.

In this paper, we establish the first generalization theory for JEPA-based world models. Based on a conditioned spectral graph formulation, we establish the equivalence between the JEPA pretraining risk and matrix factorization of the co-occurrence matrix. Then by establishing the relationship between the downstream action planning regret and the pretrained JEPA risk, we derive the generalization error bound for JEPA-based world models. 
Our bounds reveal an inherent approximation--sample error trade-off with respect to latent dimensionality.
This trade-off provides insights into the advantages and limitations of latent- versus input-level prediction, while also motivating the design of representations that preserve predictive information without unnecessarily increasing model complexity.

Our contributions are as follows.
\begin{itemize}
    \item We for the first time establish a spectral graph-based theoretical framework for JEPA-based world models, where we propose a conditioned co-occurrence matrix formulating the co-occurrence probability of the current and next state conditioned on the action, and show that the spectral JEPA objective is equivalent to a low-rank factorization of this matrix.
    \item We establish a connection between JEPA pretraining
    error and downstream planning regret in both single-
    and multi-step settings, leading to finite-sample
    generalization bounds. Our analysis reveals an inherent
    approximation--sample error trade-off with respect to
    latent dimensionality, providing theoretical insights
    into latent- versus input-level prediction.
    \item Motivated by the theoretical trade-off, we investigate a simple Multi-Depth JEPA extension that combines representations from different network depths. Through validation experiments, we demonstrate the robustness of latent-level prediction under unpredictable observations and difficult planning tasks, and show that multi-depth representations can improve planning success.
\end{itemize}

\section{Related Works}



\textbf{JEPA and JEPA-based world models.} Joint-Embedding Predictive Architectures (JEPAs) learn predictive representations without reconstructing raw observations. I-JEPA \citep{ijepa} applies this paradigm to images, while V-JEPA \citep{bardes2023v} extends it to videos through masked spatio-temporal prediction. Recent advances in world modeling include V-JEPA 2 \citep{assran2025v}, which learns action-conditioned dynamics for robotic planning, DINO-WM \citep{zhou2024dino}, which predicts future visual features for zero-shot planning, and LeWorldModel \citep{maes2026leworldmodel}, which learns a JEPA-style world model directly from pixels. Recent work has systematically examined the effects of architecture, training objectives, and planning algorithms on JEPA-based world models \citep{terver2025drives}. Other studies improve latent-space planning through goal-conditioned value learning \citep{destrade2025value}, hierarchical subgoal prediction \citep{masip2026ff}, and temporal progress supervision \citep{bai2026temporal}.

Despite the empirical success, theoretical guarantees of the JEPA-based world models are largely underexplored. 
Existing theoretical studies have investigated representation learning and regularization in JEPAs, including Gaussian-regularized objectives \citep{balestriero2025gaussian,klindt2026does}. More recent work has examined the inductive bias of object-level latent masking \citep{nam2026causal} and anti-collapse properties of unified predictive objectives \citep{lanji2026unijepa}. However, these studies do not directly characterize the finite-sample relationship between JEPA pretraining error and downstream action planning regret considered in our work.

\textbf{Spectral graph theory for representation learning.}
Spectral graph theory was introduced to self-supervised representation learning by \citet{haochen2021provable}, who build generalization guarantees for self-supervised contrastive learning by formulating the similarity of the augmented data through the concept of augmentation graph.
The theoretical framework was later extended to other representation learning paradigms, including unsupervised domain adaptation \citep{shen2022connect}, multi-modal learning \citep{zhang2023generalization}, weakly supervised learning \citep{cui2023rethinking}, autoregressive
and masked self-supervised learning (SSL) \citep{zhang2024look}, and SSL with difficult examples \citep{zhangdifficult}.
Recently, \citet{balestriero2026spectral} further brought forward a spectral graph theory of SSL that relies on harmonic analysis and spectral graph theory. 
Despite the theoretical research on representation learning, spectral graph theory has not yet been explored in JEPA-based world models which make predictions in the representation space.

\section{Mathematical Formulations}

In this section, we formulate JEPA pretraining and latent-space action planning, define downstream planning regret, and introduce the finite-sample setting for our generalization analysis.

\subsection{Pretraining}

\textbf{JEPA pretraining objective.}
In the pretraining stage of JEPA, world models are learned by forecasting future latent representations instead of reconstructing raw observations. Given a current observation $x \in \mathcal{X}:=\mathbb{R}^d$ and action $a \in \mathcal{A}$, an encoder $f: \mathbb{R}^d \to \mathbb{R}^k$ maps the observation into a latent state $z = f(x)$, while a predictor $g: \mathbb{R}^k \times \mathcal{A} \to \mathbb{R}^k$ estimates the latent representation of a future observation $x^+ \in \mathcal{X}$. The training objective minimizes the discrepancy between the predicted latent representation $\hat{z}^+ = g(f(x),a)$ and the target latent embedding $z^+ = f(x^+)$, i.e.,
\begin{align}\label{eq::loss_jepa}
	\mathcal{L}_{\mathrm{JEPA}} (x,x^+;f,g,a) 
	= \|g(f(x),a) - f(x^+)\|^2.
\end{align}
Note that the pretraining loss in \eqref{eq::loss_jepa} alone would leads to representation collapse. To prevent this, the JEPA-based models either use uniformity regularizations \citep{balestriero2025lejepa, maes2026leworldmodel} or techniques such as stop gradient and exponential moving average \citep{assran2025v}.

In this paper, for the ease of theoretical analysis, we adopt a uniformity term inspired by spectral contrastive learning \citep{haochen2021provable} \footnote{We empirically evaluate the effect of the additional uniformity term in Appendix~\ref{app:uniformity}, and observe modest but consistent improvements across all different environments.}, and define the pretraining loss as 
\begin{align}
    \mathcal{L}_{\mathrm{JEPA}} 
    = \|g(f(x),a) - f(x^+)\|^2
	+ \left[g(f(x),a)^\top f(x')\right]^2.
\end{align}

For population risk, we assume the $(x,x^+)$ pairs are generated as follows: 1) sample $x \sim \mathrm{P}_X$ from $\mathbb{R}^d$, and 2) given $x$ and $a \in \mathcal{A}$, sample $x^+ \sim \mathrm{P}(\cdot|x,a)$. Then given action $a$, we define the pretraining risk as 
\begin{align}\label{eq::def_jeparisk}
	\mathcal{R}_{\mathrm{JEPA}} (f,g,a) 
	&= \mathbb{E}_{x \sim \mathrm{P}_X} \mathbb{E}_{x^+ \sim \mathrm{P}(\cdot|x,a)} \|g(f(x),a) - f(x^+)\|^2 
	+ \mathbb{E}_{\overset{x \sim \mathrm{P}_X}{x' \sim \mathrm{P}_{X|a}}} \left[g(f(x),a)^\top f(x')\right]^2.
\end{align}

Then given action $a \in \mathcal{A}$, we denote the optimal encoder and predictor under the population pretraining risk as 
\begin{align}
    (f^*, g^*) = \arg\min_{f,g} \mathcal{R}_{\mathrm{JEPA}} (f,g,a).
\end{align}

\textbf{Conditional co-occurrence matrix and matrix factorization.}
Inspired by previous works that introduced spectral graph theory to representation learning \citep{haochen2021provable, zhang2024look}, we here introduce a conditional co-occurrence matrix that formulates the co-occurrence relationship between $(x,x^+)$ pairs conditioned on the actions.

\begin{figure}[!t]
    \vspace{-3mm}
    \centering
    \includegraphics[width=\linewidth]{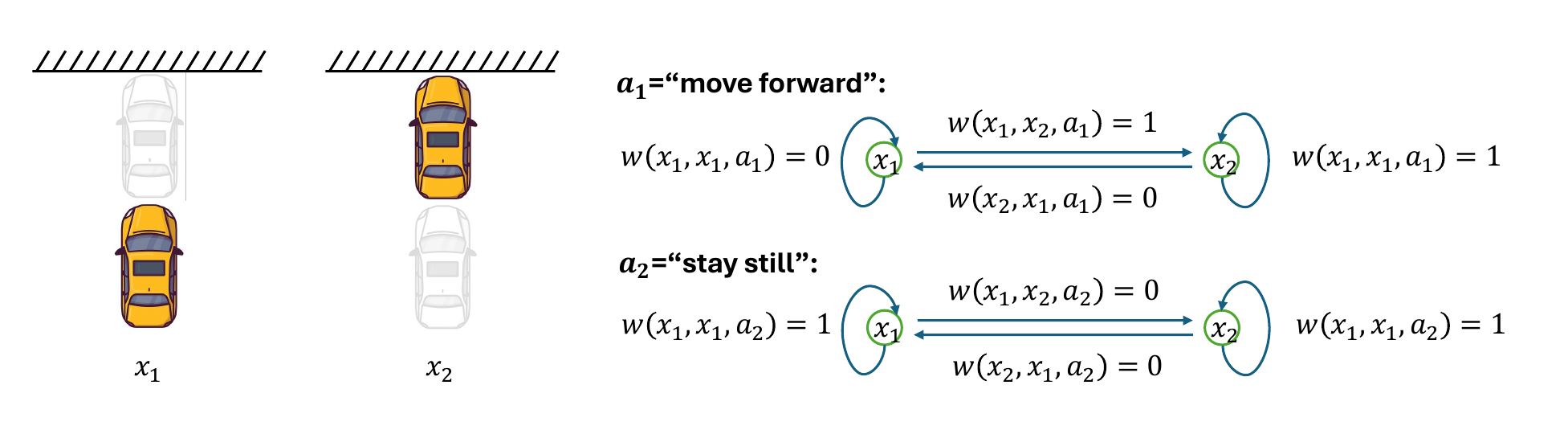}
    \vspace{-8mm}
    \caption{An illustrative example of the graph relationship between world model inputs. (Left) Observations of a car-moving example, where $x_1$ and $x_2$ represent two possible states of car locations. (Right) The transition probabilities from $x_i$ to $x_j$, $i,j \in \{1,2\}$ are conditioned on the actions $a_1$ and $a_2$. The conditioned co-occurrence matrix $M(a)$ is formed by assembling $w(x_i,x_j,a)$.}
    \label{fig::illustration}
    \vspace{-3mm}
\end{figure}

Given action $a \in \mathcal{A}$, for $x,x^+ \in \mathbb{R}^d$, we denote $w(x,x^+,a) := \mathrm{P}(x,x^+|a) \geq 0$ and denote the co-occurrence matrix 
\begin{align}
	M(a) := \left(w(x,x^+,a)\right)_{x,x^+ \in \mathbb{R}^d}.
\end{align}
Note that different from the adjacency matrix formulated in \citet{haochen2021provable}, our conditional co-occurrence matrix $M(a)$ is asymmetric because typically we have $w(x,x^+,a)\neq w(x^+,x,a)$. Also, compared with the co-occurrence matrix proposed in \citet{zhang2024look}, our conditional co-occurrence matrix has an additional conditional relationship on the action variable $a \in \mathcal{A}$.
In Figure \ref{fig::illustration}, we illustrate how a conditioned co-occurrence matrix $M(a)$ is formulated by a two-state car-moving toy example.

We denote the marginal probabilities as $w(x) := \sum_{x^+ \in \mathcal{X}} w(x,x^+,a)$ and $w(x^+|a) := \sum_{x \in \mathcal{X}} w(x,x^+,a)$. Then according to the data generation process, we have $w(x) = \mathrm{P}_X(x)$ and $w(x^+|a) = \mathrm{P}_{X|a}(x^+)$.
Moreover, we denote the normalized co-occurance matrix as $\bar{M}(a) := D^{-1/2} M(a) D_+^{-1/2}(a)$, where $D := \mathrm{diag}(w(x))_{x \in \mathcal{X}}$ and $D_+(a) := \mathrm{diag}(w(x^+|a))_{x^+ \in \mathcal{X}}$.

In Theorem \ref{thm::spectral_jepa}, we derive the equivalence between JEPA risk and the matrix factorization of $\bar{M}(a)$.
\begin{theorem}[Equivalence between JEPA  and matrix factorization]\label{thm::spectral_jepa}
	Given $a \in \mathcal{A}$, for normalized embedding functions $f(x)$ and $g(f(x),a)$, we have
	\begin{align}
		\mathcal{R}_{\mathrm{JEPA}} (f,g,a) 
		= \|\bar{M}(a) - G(F,a)^\top F\|^2 + \mathrm{const.},
	\end{align}
	where $F := \big[\sqrt{w(x|a)}f(x)\big]_{x \in \mathcal{X}}$ and $G(F,a):=\big[\sqrt{w(x)}g(f(x),a)\big]_{x \in \mathcal{X}}$.
\end{theorem}
In the following sections, for notational simplicity, we denote the matrix factorization risk as 
$\mathcal{R}_{ \mathrm{S\text{-}JEPA}} := \|\bar{M}(a) - G(F,a)^\top F\|^2$.

\subsection{Action Planning}

In the stage of action planning, we perform trajectory optimization in the latent space. For a $T$-step planning, given an initial observation $x_0$, we randomly initialize a candidate action sequence $(a_0, \ldots, a_{T-1})$ and iteratively rollout predicted latent states up to a planning horizon. Action planning is performed by optimizing the action sequence to minimize the difference between the predicted latent and the latent of goal observation $z_g = f(x_g)$. Specifically, for $t = 0,\ldots,T-1$, given predictor $g$, we denote
\begin{align}
    \hat{z}_{t+1} = g(\hat{z}_t,a_t)
\end{align}
and 
\begin{align}
    \mathcal{L}_{\mathrm{plan}}(a_0,\ldots,a_{T-1};f,g,x_0, x_g) 
    &= \|\hat{z}_T - z_g\|^2
    = \|g\left(\cdots g(f(x_0),a_0)\ldots, a_{T-1}\right) - f(x_g)\|^2.
\end{align}
Then given the optimal encoder $f^*$ and predictor $g^*$, for an initial observation $x_0$, we denote the optimal action sequence as
\begin{align}
    (a_0^*, \ldots, a_{T-1}^*) = \arg\min_{a_0,\ldots,a_{T-1}} \mathcal{L}_{\mathrm{plan}}(a_0,\ldots,a_{T-1};f^*,g^*,x_0, x_g).
\end{align}

\subsection{Evaluation}
Recall that in pretraining and action planning, the prediction is conducted in the latent space. By contrast, the planned actions are to be utilized in the input-space downstream tasks. Therefore, we evaluate the planned actions by measuring the predicted goal observation in the input space.
Specifically, given initial observation $x_0$, we obtain the $T$-step interatively rollout prediction $\hat{x}_T$ by 
$\mathrm{P}(\hat{x}_{t+1}| \hat{x}_{t}, a)$ for $t=1,\ldots, T-1$ and $\mathrm{P}(\hat{x}_1| x_0, a)$. By Bayes formula, we have $\mathrm{P}(\hat{x}_{t+1}| \hat{x}_{t}, a) = w(\hat{x}_{t},\hat{x}_{t+1},a)/w(\hat{x}_{t})$.
Then we have
\begin{align}
	V(a_0,\ldots,a_{T-1}) 
	&= \mathrm{P}(\hat{x}_T=x_g|x_0, a_0,\ldots,a_{T-1})
	\nonumber\\
	&= \sum_{x_1,\ldots,x_{T-1}} \mathrm{P}(\hat{x}_T=x_g|x_{T-1}, a_{T-1}) \cdots \mathrm{P}(\hat{x}_1=x_1|x_0, a_0)
	\nonumber\\
	&= \sum_{x_1,\ldots,x_{T-1}} \frac{w(x_{T-1},x_g,a_{T-1})}{w(x_{T-1})} \cdots \frac{w(x_0,x_1,a_0)}{w(x_0)}.
\end{align}
By denoting $a^{**} = \arg\max_a V(a)$,
we define the pointwise regret as
\begin{align}
	r(a_0,\ldots,a_{T-1})  = V(a^{**}) - V(a).
\end{align}
Then we denote the expected planning regret as
\begin{align}
	\mathcal{E}(a) = \mathbb{E}_{x_0,x_g} 	r(a_0,\ldots,a_{T-1}).
\end{align}

\subsection{Finite-Sample Optima}\label{sec::finite_sample}

To derive a finite-sample error bound for the JEPA framework, given $a \in \mathcal{A}$, we define the finite-sample optima of the encoder $f$ and predictor $g$ as 
\begin{align}
	\hat{f}, \hat{g} = \arg\min_{f,g} \hat{\mathcal{R}}_S(f,g,a),
\end{align}
where $\hat{\mathcal{R}}_S(f,g,a)$ is an unbiased estimator of $\mathcal{R}(f,g,a)$ with details specified in Appendix \ref{app::Rhat}.
Then given the initial observation $x_0$ and the goal observation $x_g$, we define the finite-sample optimal action sequence as
\begin{align}
	(\hat{a}_0,\ldots, \hat{a}_{T-1}) = \arg\min_{a_0,\ldots,a_{T-1}} \mathcal{L}_{\mathrm{plan}}(a_0,\ldots,a_{T-1};\hat{f},\hat{g},x_0, x_g).
\end{align}

\section{Generalization Error Bounds}

In this section, we derive the generalization error bounds for the downstream action planning error of JEPA-based world models.

\subsection{Evaluation Error for Action Planning}

First, we study the evaluation error for single-step and multi-step action planning. 
We note that Theorems \ref{thm::bound_T1} and \ref{thm::bound_T} are among the key theoretical results of this paper, revealing why an action series planned in the latent space generalizes well in the downstream evaluation on input-space samples.

\begin{theorem}[Single-Step Planning]\label{thm::bound_T1}
	When $T=1$, if we assume that $\mathbb{E}_{x_g}w(x_g|a) = \mathbb{E}_{x_g}w(x_g|a')$ for $a, a' \in \mathcal{A}$. For arbitrary $f$ and $g$, if we define
	\begin{align}
		\tilde{a} = \arg\min_a \mathcal{L}_{\mathrm{plan}}(a;f,g),
	\end{align}
	then we have
	\begin{align}
		\mathcal{E}(\tilde{a}) 
		&\leq 2c_0 \cdot \max_a \sqrt{\mathcal R_{ \mathrm{S\text{-}JEPA}}(f,g,a)},
	\end{align}
	where 
	$c_0 := \mathbb{E}_{x_0,x_g} \sqrt{\max_a w(x_g|a)/w(x_0)}$.
\end{theorem}
Theorem \ref{thm::bound_T1} shows that for single-step planning, the expected planning regret is upper bounded by the pretraining risk of S-JEPA. That is, if an encoder $f$ and predictor $g$ makes the pretraining risk small enough, then the action $\tilde{a}$ optimized under $f$ and $g$ is good enough for the downstream action planning task. 

\begin{theorem}[Multi-Step Planning]\label{thm::bound_T}
	When $T>1$, if we assume that $\mathbb{E}_{x_g}w(x_g|a) = \mathbb{E}_{x_g}w(x_g|a')$ for $a, a' \in \mathcal{A}$, and $M(a)$ is deterministic, i.e., given $a \in \mathcal{A}$ and $x \in \mathcal{X}$, for $x^+ \in \mathcal{X}$ there exists $w(x,x^+,a)=1$ and $w(x,x',a)=0$ for $x' \neq x^+$. Then for arbitrary $f$ and $g$, by defining
	\begin{align}
		\tilde{a}_0, \ldots, \tilde{a}_{T-1} = \arg\min_{a_0,\ldots a_{T-1}} \mathcal{L}_{\mathrm{plan}}(a_0,\ldots,a_{T-1};f,g),
	\end{align}
	we have
	\begin{align}
		\mathcal{E}(\tilde{a}_0,\ldots,\tilde{a}_T) 
		&\leq 2T c_3  \cdot \sqrt{\max_a \mathcal{R}_{ \mathrm{S\text{-}JEPA}}(f,g,a)},
	\end{align}
	where $c_3 := \max_{a,x_0,x_g}\sqrt{w(x_g|a)/w(x_0)}$.
\end{theorem}
In Theorem \ref{thm::bound_T}, we derive the evaluation error bound for multi-step action planning. Compared with Theorem \ref{thm::bound_T1}, the $T$-step planning regret is approximately $T$ times the single-step planning regret, explaining that longer planning horizon leads to higher planning error due to error accumulation. Also note that compared with Theorem \ref{thm::bound_T1}, Theorem \ref{thm::bound_T} requires a slightly stronger assumption that $\bar{M}(a)$ is deterministic. This assumption adheres to real world dynamic systems, where given a prior observation $x_t$ and action $a_t$, the next observation $x_{t+1}$ should be deterministically determined.

\subsection{Error Analysis for JEPA Pretraining Risk}
In the following, we analyze the approximation error and sample error for the S-JEPA pretraining risk respectively.
\begin{theorem}[Approximation Error for Spectral JEPA Risk]\label{thm::approx}
	Given $a \in \mathcal{A}$, the optimal population risk equals
	\begin{align}
		\mathcal R_{ \mathrm{S\text{-}JEPA}}(f^*,g^*,a)
		=\sum_{i>k}\sigma_i^2(a),
	\end{align}
	where $\sigma_i(a)$ is the $i$-th largest singular value of $\bar{M}(a)$.
\end{theorem}

In Theorem \ref{thm::approx}, we show that as the approximation error of S-JEPA risk is determined by the singular values of the conditioned co-occurrence matrix. As the latent dimension $k$ increases, the approximation error decreases.

\begin{theorem}[Sample Error Bound for Spectral JEPA Risk]\label{thm::sample_error}
	Assume that $\|f\|_{\infty}<\kappa$ and $\|g\|_{\infty}<\kappa$ for all
	$f\in\mathcal F$ and $g\in\mathcal G$. Fix $a\in\mathcal A$. 
	Then, with probability at least $1-\delta$, we have
	\begin{align}
		&\mathcal R_{\mathrm{S\text{-}JEPA}}(\hat f,\hat g,a)
		-
		\mathcal R_{\mathrm{S\text{-}JEPA}}(f^*,g^*,a)
		\nonumber\\
		&\le
		c_1
		\left[
		\mathfrak R_{n/2}(\mathcal G\circ \mathcal F)
		+
		\mathfrak R_{n/2}(\mathcal F)
		\right]
		+
		c_2
		\left(
		\sqrt{\frac{\log(2/\delta)}{n}}
		+
		\delta
		\right),
	\end{align}
	where
$c_1 :=32k^2\kappa^3+32k\kappa$,
$c_2:=8k\kappa^2+2k^2\kappa^4$, and 
\begin{align*}
	\mathfrak R_{n}(\mathcal G\circ \mathcal F) &= \max_{(z_i, z'_i, z_i^+)_{i=1}^n} \mathbb E_{\sigma}
	\left[
	\sup_{f\in\mathcal F,\;g\in\mathcal G}
	\frac1m
	\sum_{i=1}^m
	\sigma_i \Big(
	\big(g(f(z_i),a)^\top f(z_i')\big)^2 - 2g(f(z_i),a)^\top f(z_i^+)\Big)
	\right],
	\\
	&\mathfrak R_{n}(\mathcal F) = \max_k \max_{\{z_i\}_{i=1}^n} \mathbb E_{\sigma}
	\left[
	\sup_{f\in\mathcal F,\;g\in\mathcal G}
	\frac1m
	\sum_{i=1}^m
	\sigma_i f_k(z_i)
	\right].
\end{align*}
\end{theorem}

In Theorem \ref{thm::sample_error}, we show the sample error bound of S-JEPA risk, which depends on the Rademacher complexities $\mathfrak{R}_n(\mathcal{G}\circ \mathcal{F})$ and $\mathfrak{R}_n(\mathcal{F})$. Note that as the latent dimension $k$ increases, both the output dimensions of $g \circ f \in \mathcal{G}\circ \mathcal{F}$ and $f \in \mathcal{F}$ increase, so typically we have larger Rademacher complexities and accordingly larger sample error.

\subsection{Finite-Sample Error Bounds for Action Planning}

Based on Theorems in the previous sections, we can now derive the finite-sample generalization error bound for action planning in Theorems \ref{thm::generalization} and \ref{thm::generalizationT}.
\begin{theorem}[Finite-Sample Error Bound for Single-Step Planning]\label{thm::generalization}
	For some $\kappa>0$, assume $\|g(f(x),a)\|_\infty \leq \kappa$ and $\|f(x)\|_\infty \leq \kappa$ for all $f$, $g$, and $a$. Assume that all $\mathcal{R}_{n/2}(\mathcal{G}\circ \mathcal{F})$ are of the same order regardless of $a$. Then for $T=1$, with probability at least $1-\delta$, we have 
	\begin{align}
		\mathcal{E}(\hat{a})
		&\leq 2c_0 \cdot \sqrt{\underbrace{\max_a \sum_{i>k} \sigma_i^2(a)}_{Approximaition\ Error} + \underbrace{c_1 \left[\mathfrak{R}_n(\mathcal{G}\circ \mathcal{F})\mathfrak{R}_n(\mathcal{F})\right] + c_2 \cdot \left(\sqrt{\frac{\log 2/\delta}{n}} + \delta\right)}_{Sample\ Error} +c},
	\end{align}
	where $c_3 := \max_{a,x_0,x_g} \sqrt{w(x_g|a)/w(x_0)}$, $c_1 := 16k^2\kappa^2+16k^2\kappa$, $c_2 := 8k\kappa^2 + 2k^2\kappa^4$, and $c := 2 - \mathbb{E}_{x_0,x_g} \max_a \frac{w(x_0,x_g,a)^2}{w(x_0)w(x_g|a)}$.
\end{theorem}

In Theorem \ref{thm::generalization}, we show the finite-sample error bound for single-step planning. 
Given sample size $n$, there is a trade-off between the approximation error term and sample error term with respect to the latent dimension $k$. That is, larger $k$ leads to smaller $\max_a\sum_{i>k} \sigma_i^2(a)$ but larger $c_1$, $c_2$, and $\mathfrak{R}_n(\mathcal{G}\circ \mathcal{F}) + \mathfrak{R}_n(\mathcal{F})$. 
Moreover, as input-level predictive models can be viewed as special cases of the latent level ones, where $k=n$ and $f$ degenerates to an identity mapping. In this case, we have zero approximation error but the largest sample error.

\begin{theorem}[Finite-Sample Error Bound for Multi-Step Planning]\label{thm::generalizationT}
	For some $\kappa>0$, assume $\|g(f(x),a)\|_\infty \leq \kappa$ and $\|f(x)\|_\infty \leq \kappa$ for all $f$, $g$, and $a$, and $M(a)$ is deterministic. Assume that all $\mathcal{R}_{n/2}(\mathcal{G}\circ \mathcal{F})$ are of the same order regardless of $a$. Then for $T>1$, with probability at least $1-\delta$, we have 
	\begin{align}
		\mathcal{E}(\hat{a})
		&\leq 2c_3 T\cdot \sqrt{\underbrace{\max_a \sum_{i>k} \sigma_i^2(a)}_{Approximaition\ Error} + \underbrace{c_1 \left[\mathfrak{R}_n(\mathcal{G}\circ \mathcal{F})
		+ \mathfrak{R}_n(\mathcal{F})\right]
		+ c_2 \cdot \left(\sqrt{\frac{\log 2/\delta}{n}} + \delta\right)}_{Sample\ Error} +c},
	\end{align}
	where $c_0 := \mathbb{E}_{x_0,x_g} \sqrt{\max_a w(x_g|a)/w(x_0)}$, $c_1 := 16k^2\kappa^2+16k^2\kappa$, $c_2 := 8k\kappa^2 + 2k^2\kappa^4$, and $c := 2 - \mathbb{E}_{x_0,x_g} \max_a \frac{w(x_0,x_g,a)^2}{w(x_0)w(x_g|a)}$.
\end{theorem}

In Theorem \ref{thm::generalizationT}, we show the finite-sample error bound for multi-step action planning. The error trade-off w.r.t. the latent dimension $k$ is similar to the single-step planning case.

\subsection{Approximation and Sample Error Trade-off}
The above bounds reveal an intrinsic trade-off between the \emph{Approximation Error} and the \emph{Sample Error}, governed by the latent dimension $k$. On the one hand, a smaller latent dimension imposes a stronger information bottleneck, which reduces the complexity of the learned encoder and predictor and therefore leads to a smaller Sample Error. However, such compression may discard action-relevant transition information, resulting in a larger Approximation Error. On the other hand, increasing $k$ allows the latent model to preserve more singular components of the action-conditioned co-occurrence matrix, thereby reducing the Approximation Error. This benefit comes at the cost of a larger hypothesis space and higher sample complexity, which increases the Sample Error. Therefore, latent predictive models are most advantageous when a moderate-dimensional representation can retain the task-relevant dynamics while filtering out nuisance or unpredictable input-level variations. In contrast, input-level predictive models can be viewed as an extreme case with minimal Approximation Error but potentially maximal Sample Error, since they are required to model all observation dimensions, including those irrelevant to downstream planning.

\subsection{A Theory-Motivated Extension: Multi-Depth JEPA}
\label{sec:md-jepa}

The analysis above highlights the importance of the latent representation in balancing approximation and sample errors. While increasing the latent dimension can reduce the approximation error by preserving more components of the action-conditioned co-occurrence matrix, simply enlarging the representation does not necessarily introduce additional action-relevant information. This motivates us to investigate whether the representational capacity of a latent world model can be increased by incorporating complementary information from different depths of the model.

To this end, we introduce \textit{Multi-Depth JEPA (MD-JEPA)}, which constructs the latent representation by concatenating intermediate representations from multiple layers of the world model. Let $f_l(x) \in \mathbb{R}^{k_l}$ denote the representation of observation $x$ at layer $l$. Given a set of selected layers $\mathcal{L}$, we define the multi-depth latent representation as
\begin{equation}
    f_{\mathcal{L}}(x)
    =
    \left[
    f_l(x)
    \right]_{l \in \mathcal{L}}
    \in
    \mathbb{R}^{k_{\mathcal{L}}},
    \qquad
    k_{\mathcal{L}}
    =
    \sum_{l \in \mathcal{L}} k_l .
\end{equation}
The action-conditioned predictor then operates on the concatenated representation and predicts the corresponding future representation:
\begin{equation}
    \hat{z}_{\mathcal{L}}^{+}
    =
    g_{\mathcal{L}}
    \left(
    f_{\mathcal{L}}(x), a
    \right),
\end{equation}
where the prediction target is given by $f_{\mathcal{L}}(x^{+})$. Accordingly, the latent prediction objective becomes
\begin{equation}
    \mathcal{L}_{\mathrm{MD-JEPA}}^{\mathcal{L}}
    =
    \left\|
    g_{\mathcal{L}}
    \left(
    f_{\mathcal{L}}(x), a
    \right)
    -
    f_{\mathcal{L}}(x^{+})
    \right\|_2^2,
\end{equation}
together with the same anti-collapse regularization as in the standard JEPA formulation.

MD-JEPA provides an alternative way to increase the effective latent dimension. Rather than expanding a representation within a single layer, it incorporates features from different depths of the world model, which may contain complementary information relevant to predicting future states. Importantly, increasing the representation dimension alone does not imply that additional information is obtained.



\section{Validation Experiments}
\label{sec:experiments}

We conduct two complementary experiments to validate the implications of our theoretical analysis. 
First, we use a fully controlled continuous-control environment to compare latent-level prediction with input-level reconstruction, and examine how their relative performance changes with observation noise and planning difficulty. 
Second, motivated by our analysis of latent representation capacity, we investigate whether incorporating representations from multiple model depths can further improve planning performance.

\subsection{Latent- vs. Input-Level Prediction}
\label{sec:latent-input}

\begin{figure}[t]
    \centering
    \vspace{-3mm}
    \includegraphics[width=\linewidth]{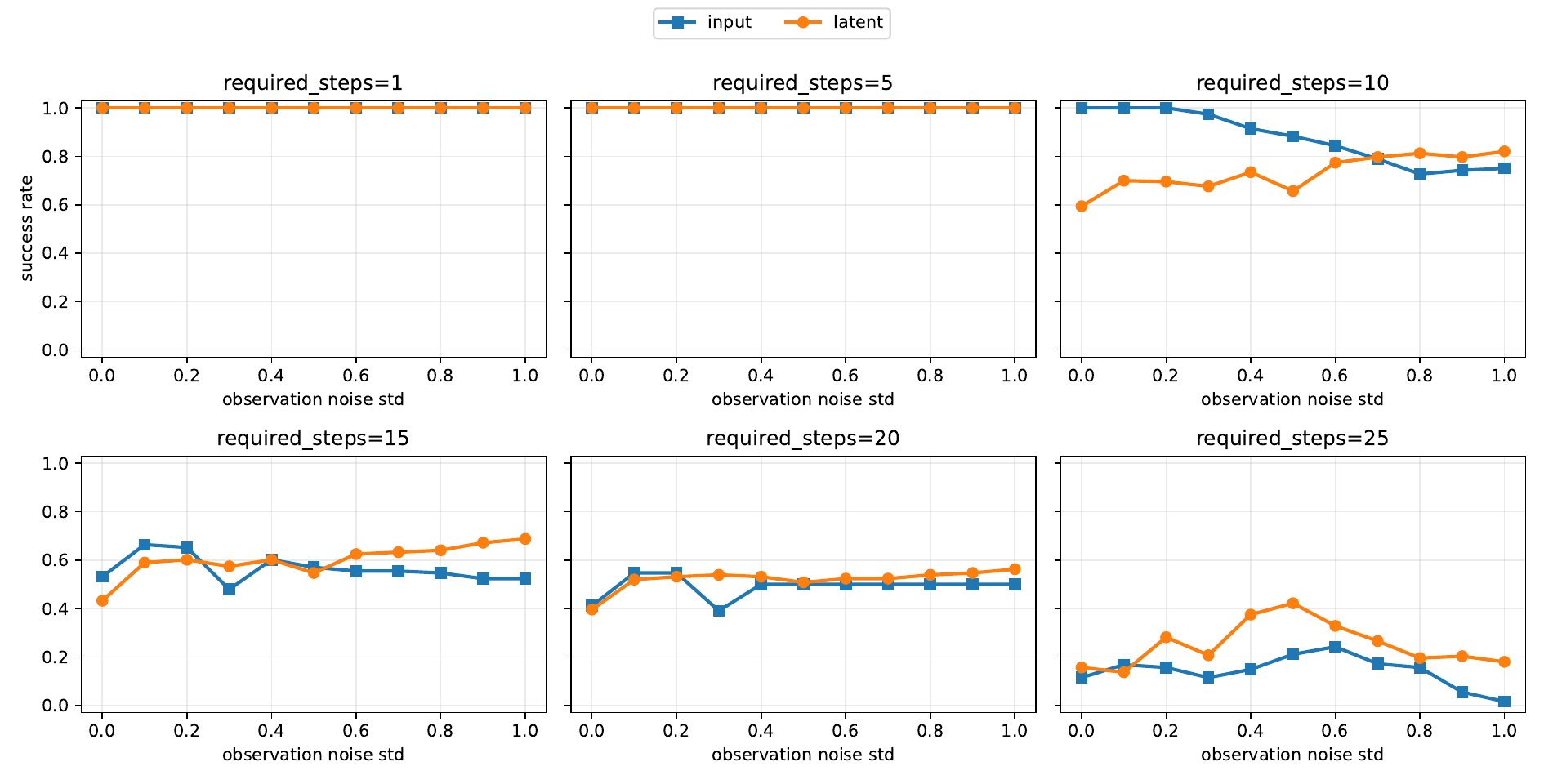}
    \vspace{-3mm}
    \caption{
    Success rates of latent- and input-level predictive models under different noise levels and required planning steps.
    Both approaches perform similarly on short and easy planning tasks, while latent-level prediction becomes more robust against increasingly unpredictable observations and more difficult long-horizon planning.
    }
    \label{fig::synthetic_latent_input}
    \vspace{-3mm}
\end{figure}

\paragraph{Experimental setup.}
We construct a controlled 2D point-mass environment to compare latent-level prediction with input-level reconstruction. The underlying state is defined by the 2D position and velocity of the point mass, $s_t=[p_x,p_y,v_x,v_y]$, and the action $a_t=[a_x,a_y]$ specifies the applied acceleration along the two spatial dimensions. The model does not directly observe this state, but instead receives a higher-dimensional observation composed of task-relevant state features, nonlinear nuisance features, and random noise features that are independent of the action and downstream objective.

Both methods use the same action-conditioned predictive backbone and differ only in the prediction target: the latent-level model predicts the future representation $z_{t+H}$, whereas the input-level model predicts the future observation $x_{t+H}$. We use a fixed prediction horizon of $H=3$. For downstream evaluation, we consider a goal-reaching task with receding-horizon CEM planning. At the beginning of each episode, we sample a start state and a goal state, where the goal specifies the target position to be reached and is provided to the planner through its corresponding observation $x_{\mathrm{goal}}$. For latent-level planning, the goal is encoded as $z_{\mathrm{goal}}=\mathrm{Enc}(x_{\mathrm{goal}})$. The latent-level planner evaluates candidate action sequences in latent space, whereas the input-level planner evaluates them in observation space.

In Figure~\ref{fig::synthetic_latent_input}, we vary two factors. The \textit{noise level} controls the contribution of unpredictable noise dimensions to the input reconstruction objective, while the \textit{required planning steps} indicates approximately how many closed-loop planning decisions are needed to reach the goal. Larger values therefore correspond to stronger interference from unpredictable observations and more difficult long-horizon planning, respectively. Additional environment, training, and planning details are provided in Appendix~\ref{app:synthetic}.


\paragraph{Results.}
Figure~\ref{fig::synthetic_latent_input} shows that the two approaches perform similarly on short and easy planning tasks. As the required number of planning steps and the influence of unpredictable observation noise increase, latent-level prediction becomes increasingly more robust. The difference is most pronounced at 25 required planning steps, where input-level prediction degrades substantially while latent-level prediction maintains noticeably higher success rates.

These observations are consistent with our theoretical analysis: latent prediction can suppress unpredictable observation components that are irrelevant to downstream control, and this advantage becomes more important as prediction errors accumulate over repeated long-horizon planning steps.

\subsection{Empirical Evaluation of Multi-Depth Representations}
\label{sec:md-jepa-exp}

\paragraph{Experimental setup.}
We evaluate MD-JEPA introduced in Section~\ref{sec:md-jepa} on PushT using DINO-WM with a ViT backbone, considering different combinations of intermediate representations. We use Layer 12 as the standard single-depth baseline and consider three multi-depth variants, $3,12$, $6,12$, and $9,12$. To control for the increase in representation dimensionality, we additionally include Layer 12 $\times 2$, which duplicates the Layer-12 representation without introducing information from another depth. All other training and planning configurations are kept unchanged. We evaluate planning performance at horizons $H=25$, $50$, and $75$.

\begin{table}[t]
\centering
\vspace{-3mm}
\caption{
Planning success rates (\%) of different latent representations at different horizons.
Layer 12 $\times 2$ serves as a dimensionality-matched control for MD-JEPA.
}
\label{tab:md-jepa}
\begin{tabular}{lccc}
\toprule
Representation & $H=25$ & $H=50$ & $H=75$ \\
\midrule
Layer 12
    & 64.0 & 24.7 & 6.7 \\
Layer 12 $\times$ 2
    & 68.0 & 23.3 & 8.0 \\
MD-JEPA $[6,12]$
    & 74.7 & \textbf{33.3} & \textbf{14.0} \\
MD-JEPA $[3,12]$
    & \textbf{77.3} & 24.7 & 10.0 \\
MD-JEPA $[9,12]$
    & 66.7 & 28.0 & 10.0 \\
\bottomrule
\end{tabular}
\vspace{-3mm}
\end{table}

\paragraph{Results.}
As shown in Table~\ref{tab:md-jepa}, MD-JEPA improves planning performance over the standard Layer-12 representation across different horizons.
At $H=25$, the $[3,12]$ combination achieves the highest success rate of $77.3\%$, compared with $64.0\%$ for Layer 12.
At longer horizons, $[6,12]$ performs best, improving the success rate from $24.7\%$ to $33.3\%$ at $H=50$ and from $6.7\%$ to $14.0\%$ at $H=75$.

The dimensionality-matched control further shows that the gains cannot be attributed solely to a larger latent dimension.
Duplicating Layer 12 provides only limited improvement, whereas incorporating representations from different depths yields substantially larger gains, especially at longer horizons.
These results suggest that representations from different depths provide complementary predictive information that cannot be obtained by simply enlarging the final-layer representation.

\section{Conclusion}
We established a generalization theory for JEPA-based world models by showing that the spectral JEPA objective is equivalent to a low-rank factorization of an action-conditioned co-occurrence matrix. We further derived finite-sample planning regret bounds, revealing an approximation--sample error trade-off with respect to latent dimensionality. Experiments demonstrate the robustness of latent-level prediction under unpredictable observations and difficult planning tasks. Motivated by our theory, we also investigated Multi-Depth JEPA, which improves planning success over single-depth representations and a dimensionality-matched control, illustrating how theoretical insights can inform representation design.

\bibliography{arxiv_jepa}
\bibliographystyle{iclr2027_conference}

\newpage
\appendix
\section{Appendix}
\subsection{Detailed Finite-Sample Formulation}\label{app::Rhat}

We define the empirical JEPA risk in Definition \ref{def::empirical_jepa}.

\begin{definition}[Empirical JEPA risk]\label{def::empirical_jepa}
	Consider a dataset $\hat{\mathcal{X}} = \{x_1, \ldots ,x_n\}$ containing $n$ data points i.i.d. sampled from $\mathrm{P}_X$. Let $\hat{\mathrm{P}}_X$ be the uniform distribution over $\hat{\mathcal{X}}$. Let $\hat{\mathrm{P}}_{x,x'}$ be the uniform distribution over data pairs $(x_i,x_j)$ where $i\neq j$. We define the empirical spectral JEPA loss of a feature extractor $f$ as 
	\begin{align}
		\hat{\mathcal{R}}_n(f,g,a) 
		&:= -2 \mathbb{E}_{x \in \hat{\mathrm{P}}_X, x^+ \sim \mathrm{P}(\cdot|x,a)} \left[g(f(x),a)^\top f(x')\right]
		\nonumber\\
		&+ \mathbb{E}_{x, x^- \in \hat{\mathrm{P}}_{x,x^-}, x' \sim \mathrm{P}(\cdot|x^-,a)} \left[g(f(x),a)^\top f(x')\right]^2.
	\end{align}
\end{definition}

Lemma \ref{lem::unbias1} shows that $\hat{\mathcal{R}}_n(f) $ is an unbiased estimator of population spectral JEPA loss.
\begin{lemma}\label{lem::unbias1}
	$\hat{\mathcal{R}}_n(f,g,a)$ is an unbiased estimator of $\mathcal{R}(f,g,a)$, i.e.,
	\begin{align}
		\mathbb{E}_{\hat{\mathcal{X}}}\left[\hat{\mathcal{R}}_n(f,g,a) \right] = \mathcal{R}(f,g,a).
	\end{align}
\end{lemma}

\begin{definition}
	Given dataset $\hat{\mathcal{X}}$, we sample a subset of tuples as follows: first sample a permutation $\pi : [n] \to [n]$. Then given $a \in \mathcal{A}$, we sample tuple $S = \{(z_i,z_i^+,z_i')\}_{i=1}^{n/2}$ as follows:
	\begin{align}
		z_i = x_{\pi_{2i-1}}, \qquad z_i^+ \sim \mathrm{P}(\cdot|z_i,a), \qquad z_i' \sim \mathrm{P}(\cdot|x_{\pi_{2i}},a).
	\end{align} 
	We define the following loss on $S$:
	\begin{align}
		\hat{\mathcal{R}}_S(f,g,a) = \frac{1}{n/2}\sum_{i=1}^{n/2}\left[\big(g(f(z_i),a)^\top f(z_i')\big)^2 - 2g(f(z_i),a)^\top f(z_i^+)\right].
	\end{align}
\end{definition}

In Lemma \ref{lem::unbias2}, we see that $\hat{\mathcal{R}}_S(f,g,a)$ is an unbiased estimator of $\hat{\mathcal{R}}_n(f,g,a)$.
\begin{lemma}\label{lem::unbias2}
	Given $\hat{\mathcal{X}}$, we have
	\begin{align}
		\mathbb{E}_S \hat{\mathcal{R}}_S(f,g,a) = \hat{\mathcal{R}}_n(f,g,a).
	\end{align}
\end{lemma}

\subsection{Proofs}

\begin{proof}[Proof of Theorem \ref{thm::spectral_jepa}]
    By the definition of JEPA risk in \eqref{eq::def_jeparisk}, given $a \in \mathcal{A}$, we have
    \begin{align}
        &\mathcal{R}_{\mathrm{JEPA}} (f,g,a) 
    	\nonumber\\
        &= \mathbb{E}_{x \sim \mathrm{P}_X} \mathbb{E}_{x^+ \sim \mathrm{P}(\cdot|x,a)} \|g(f(x),a) - f(x^+)\|^2 
    	+ \mathbb{E}_{\overset{x \sim \mathrm{P}_X}{x' \sim \mathrm{P}_{X|a}}} \left[g(f(x),a)^\top f(x')\right]^2
        \nonumber\\
        &= \sum_{x,x^+} w(x,x^+,a) \|g(f(x),a) - f(x^+)\|^2 + \sum_{x,x'} w(x) w(x'|a) \left[g(f(x),a)^\top f(x')\right]^2
        \nonumber\\
        &= 2 - 2\sum_{x,x^+} w(x,x^+,a) g(f(x),a)^\top f(x^+) + \sum_{x,x^+} w(x) w(x^+|a) \left[g(f(x),a)^\top f(x^+)\right]^2
        \nonumber\\
        &= \sum_{x,x^+} \left[\frac{w(x,x^+,a)}{\sqrt{w(x) w(x^+|a)}} - \left[\sqrt{w(x)}g(f(x),a)\right]^\top \left[\sqrt{w(x^+|a)}f(x^+)\right]\right]^2 
        + 2 - \frac{w(x,x^+,a)^2}{w(x) w(x^+|a)}
    \end{align}
    where the third equation holds because $\|g(f(x),a)\|^2 = \|f(x^+)\|^2 = 1$. Then by denoting $F := \big[\sqrt{w(x|a)}f(x)\big]_{x \in \mathcal{X}}$ and $G(F,a):=\big[\sqrt{w(x)}g(f(x),a)\big]_{x \in \mathcal{X}}$, we have
    \begin{align}
        \mathcal{R}_{\mathrm{JEPA}} (f,g,a) 
        &= \|\bar{M}(a) - G(F,a)^\top F\|^2 + \mathrm{const.}
    \end{align}
\end{proof}

\begin{proof}[Proof of Theorem \ref{thm::bound_T1}]
	If $T=1$, by definition, we have
	\begin{align}
		\tilde{a} = \arg\min_a \|g(f(x_0),a) - f(x_g)\|^2
		= \arg\max_a g(f(x_0),a)^\top f(x_g),
	\end{align}
	and
	\begin{align}
		a^{**} = \arg\max_a V(a) = \arg\max_a w(x_g,x_0,a).
	\end{align}
	Then for given $x_0$ and $x_g$, we have
	\begin{align}
		&r(\tilde{a}; x_0, x_g) 
		\nonumber\\
		&= \frac{w(x_0,x_g,a^{**})}{w(x_0)} - \frac{w(x_0,x_g,\tilde{a})}{w(x_0)}
		\nonumber\\
		&= \frac{1}{w(x_0)}\left[w(x_0,x_g,a^{**}) - w(x_0,x_g,\tilde{a})\right]
		\nonumber\\
		&= \frac{1}{w(x_0)} \big[w(x_0,x_g,a^{**}) - w(x_0)w(x_g|a^{**}) g(f(x_0),a^{**})^\top f(x_g)
		\nonumber\\
		&\qquad\qquad+ w(x_0)w(x_g|a^{**}) g(f(x_0),a^{**})^\top f(x_g) - w(x_0)w(x_g|\tilde{a}) g(f(x_0),\tilde{a})^\top f(x_g)
		\nonumber\\
		&\qquad\qquad+ w(x_0)w(x_g|\tilde{a}) g(f(x_0),\tilde{a})^\top f(x_g) - w(x_0,x_g,\tilde{a})\big].
		\nonumber\\
		&= \sqrt{\frac{w(x_g|a^{**})}{w(x_0)}} \bigg[\frac{w(x_0,x_g,a^{**})}{\sqrt{w(x_0)w(x_g|a^{**})}} - \sqrt{w(x_0)w(x_g|a^{**})} g(f(x_0),a^{**})^\top f(x_g)\bigg]
		\label{eq::decomp_r1}\\
		&+ \big[w(x_g|a^{**}) - w(x_g|\tilde{a})\big] g(f(x_0),a^{**})^\top f(x_g)
		\label{eq::decomp_r2}\\
		&+ w(x_g|\tilde{a}) \big[g(f(x_0),a^{**})^\top f(x_g)
		- g(f(x_0),\tilde{a})^\top f(x_g)\big]
		\label{eq::decomp_r3}\\
		&- \sqrt{\frac{w(x_g|\tilde{a})}{w(x_0)}} \bigg[\frac{w(x_0,x_g,\tilde{a})}{\sqrt{w(x_0)w(x_g|\tilde{a})}} - \sqrt{w(x_0)w(x_g|\tilde{a})} g(f(x_0),\tilde{a})^\top f(x_g)\bigg].
		\label{eq::decomp_r4}
	\end{align}
	Denote $\delta(x_0,x_g,a):= \frac{w(x_0,x_g,a)}{\sqrt{w(x_0)w(x_g|a)}} - \sqrt{w(x_0)w(x_g|a)} g(f(x_0),a)^\top f(x_g)$, then we have 
	\begin{align}
		(\ref{eq::decomp_r1}) &= \sqrt{w(x_g|a^{**})/w(x_0)} \delta(x_0,x_g,a^{**})
	\end{align}
	and 
	\begin{align}
		(\ref{eq::decomp_r4}) = -\sqrt{w(x_g|\tilde{a})/w(x_0)} \delta(x_0,x_g,\tilde{a}).
	\end{align}

	Recall that $\mathcal{R}_{ \mathrm{S\text{-}JEPA}}(f,g,a) = \sum_{x_0,x_g} \delta(x_0,x_g,a)^2$. Then we have
	\begin{align}
		r(\tilde{a}; x_0, x_g) &\leq \sqrt{w(x_g|a^{**})/w(x_0)} \cdot \delta(x_0,x_g,a^{**})
		+ \sqrt{w(x_g|\tilde{a})/w(x_0)} \cdot \delta(x_0,x_g,\tilde{a})
		\nonumber\\
		&\quad+ \left[w(x_g|a^{**}) - w(x_g|\tilde{a})\right] g(f(x_0),a^{**})^\top f(x_g)
		\nonumber\\
		&\leq \sqrt{w(x_g|a^{**})/w(x_0) \cdot \mathcal{R}_{ \mathrm{S\text{-}JEPA}}(f,g,a^{**})}
		+ \sqrt{w(x_g|\tilde{a})/w(x_0) \cdot \mathcal{R}_{ \mathrm{S\text{-}JEPA}}(f,g,\tilde{a})}
		\nonumber\\
		&\quad+ w(x_g|a^{**}) - w(x_g|\tilde{a})
		\nonumber\\
		&\leq 2\sqrt{\max_a w(x_g|a)/w(x_0) \cdot \mathcal{R}_{ \mathrm{S\text{-}JEPA}}(f,g,a)}
		+ w(x_g|a^{**}) - w(x_g|\tilde{a})
		\nonumber\\
		&\leq 2\sqrt{\max_a w(x_g|a)/w(x_0)} \cdot \sqrt{\max_a \mathcal{R}_{ \mathrm{S\text{-}JEPA}}(f,g,a)}
		+ w(x_g|a^{**}) - w(x_g|\tilde{a}).
	\end{align}
	Then if $\mathbb{E}_{x_g}w(x_g|a) = \mathbb{E}_{x_g}w(x_g|a')$ for $a, a' \in \mathcal{A}$, we have
	\begin{align}
		\mathcal{E}(a^*) &= \mathbb{E}_{x_0,x_g} r(\tilde{a}; x_0, x_g)
		\leq 2 \mathbb{E}_{x_0,x_g} \sqrt{\max_a w(x_g|a)/w(x_0)} \cdot \max_a \sqrt{\mathcal{R}_{ \mathrm{S\text{-}JEPA}}(f,g,a)}.
	\end{align}
\end{proof}

\begin{proof}[Proof of Theorem \ref{thm::bound_T}]
	By definition, given $x_0$ and $x_g$, we have
	\begin{align}
		&r(\tilde{a}_0,\ldots,\tilde{a}_{T-1})  
		\nonumber\\
		&= V(a^{**}) - V(\tilde{a})
		\nonumber\\
		&=  \mathrm{P}(\hat{x}_T=x_g|x_0, a_0^{**},\ldots,a_{T-1}^{**}) - \mathrm{P}(\hat{x}_T=x_g|x_0, \tilde{a}_0,\ldots,\tilde{a}_{T-1})
		\nonumber\\
		&= \sum_{x_1,\ldots,x_{T-1}} \frac{w(x_{T-1},x_g,a_{T-1}^{**})}{w(x_{T-1})} \cdots \frac{w(x_0,x_1,a_0^{**})}{w(x_0)} 
		- \frac{w(x_{T-1},x_g,\tilde{a}_{T-1})}{w(x_{T-1})} \cdots \frac{w(x_0,x_1,\tilde{a}_0)}{w(x_0)} 
		\nonumber\\
		&= \sum_{x_1,\ldots,x_{T-1}} \frac{w(x_{T-1},x_g,a_{T-1}^{**})}{w(x_{T-1})} \cdots \frac{w(x_0,x_1,a_0^{**})}{w(x_0)} 
		- \frac{w(x_{T-1},x_g,a_{T-1}^{**})}{w(x_{T-1})} \cdots \frac{w(x_0,x_1,\tilde{a}_0)}{w(x_0)} 
		\nonumber\\
		&\qquad\qquad+\cdots
		\nonumber\\
		&\qquad\qquad+ \frac{w(x_{T-1},x_g,\hat{a}_{T-1}^{*})}{w(x_{T-1})} \cdots \frac{w(x_0,x_1,\hat{a}_0^{*})}{w(x_0)} 
		- \frac{w(x_{T-1},x_g,\hat{a}_{T-1})}{w(x_{T-1})} \cdots \frac{w(x_0,x_1,\hat{a}_0)}{w(x_0)} 
		\nonumber\\
		&= \sum_{x_1,\ldots,x_{T-1}}
		\prod_{t=0}^{T-1} \frac{w(x_t,x_{t+1},a_t^{**})}{w(x_t)}
		- \prod_{t=0}^{T-1} \frac{w(x_t,x_{t+1},\tilde{a}_t)}{w(x_t)}
		\nonumber\\
		&= \sum_{x_1,\ldots,x_{T-1}} \sum_{\ell=0}^{T-1}
		\prod_{t=0}^\ell \frac{w(x_t,x_{t+1},a_t^{**})}{w(x_t)} \cdot \prod_{t=\ell+1}^{T-1} \frac{w(x_t,x_{t+1},\tilde{a}_t)}{w(x_t)}
		- \prod_{t=0}^{\ell-1} \frac{w(x_t,x_{t+1},a_t^{**})}{w(x_t)} \cdot \prod_{t=\ell}^{T-1} \frac{w(x_t,x_{t+1},\tilde{a}_t)}{w(x_t)}
		\nonumber\\
		&= \sum_{x_1,\ldots,x_{T-1}} \sum_{\ell=0}^{T-1} 
		\prod_{t=0}^{\ell-1} \frac{w(x_t,x_{t+1},a_t^{**})}{w(x_t)} \cdot \left[\frac{w(x_\ell,x_{\ell+1},a_\ell^{**})}{w(x_\ell)} - \frac{w(x_\ell,x_{\ell+1},\tilde{a}_\ell)}{w(x_\ell)}\right] \cdot \prod_{t=\ell+1}^{T-1} \frac{w(x_t,x_{t+1},\tilde{a}_t)}{w(x_t)}.
	\end{align}
	By Theorem \ref{thm::bound_T1}, we have for $\ell=0,\ldots,T-1$
	\begin{align}
		\frac{w(x_\ell,x_{\ell+1},a_\ell^{**})}{w(x_\ell)} - \frac{w(x_\ell,x_{\ell+1},\tilde{a}_\ell)}{w(x_\ell)}
		&\leq 2\sqrt{\max_a w(x_{\ell+1}|a)/w(x_\ell)} \cdot \sqrt{\max_a \mathcal{R}_{ \mathrm{S\text{-}JEPA}}(f,g,a)}.
	\end{align}
	Then if $M(a)$ is deterministic, i.e., given $a \in \mathcal{A}$ and $x \in \mathcal{X}$, for $x^+ \in \mathcal{X}$ there exists $w(x,x^+,a)=1$ and $w(x,x',a)=0$ for $x' \neq x^+$, we have
	\begin{align}
	    &r(\tilde{a}_0,\ldots,\tilde{a}_{T-1}) 
		\nonumber\\
		&\leq \sum_{x_1,\ldots,x_{T-1}} \sum_{\ell=0}^{T-1} 
		\prod_{t=0}^{\ell-1} \frac{w(x_t,x_{t+1},a_t^{**})}{w(x_t)} \cdot \prod_{t=\ell+1}^{T-1} \frac{w(x_t,x_{t+1},\tilde{a}_t)}{w(x_t)}
		\nonumber\\
		&\qquad\qquad\qquad\quad\cdot 2\sqrt{\max_a w(x_{\ell+1}|a)/w(x_\ell)} \cdot \sqrt{\max_a \mathcal{R}_{ \mathrm{S\text{-}JEPA}}(f,g,a)}
		\nonumber\\
		&= \sum_{\ell=0}^{T-1} \sum_{x_1,\ldots,x_{T-1}}
		\prod_{t=0}^{\ell-2} \frac{w(x_t,x_{t+1},a_t^{**})}{w(x_t)} \cdot \prod_{t=\ell+1}^{T-1} \frac{w(x_t,x_{t+2},\tilde{a}_t)}{w(x_t)}
		\nonumber\\
		&\cdot \frac{w(x_{\ell-1},x_\ell,a_{\ell-1}^{**})}{w(x_{\ell-1})}
		\cdot \frac{w(x_{\ell+1},x_{\ell+2},a_{\ell+1}^{**})}{w(x_{\ell+1})}
		\cdot 2\sqrt{\max_a w(x_{\ell+1}|a)/w(x_\ell)} \cdot \sqrt{\max_a \mathcal{R}_{ \mathrm{S\text{-}JEPA}}(f,g,a)}
		\nonumber\\
		&\leq \sum_{\ell=0}^{T-1} \sum_{x_\ell} \mathrm{P}(x_{\ell}| x_0, a_0^{**},\ldots,a_{\ell-1}^{**}) \mathrm{P}(x_g| x_{\ell}, \tilde{a}_{\ell}, \ldots, \tilde{a}_{T-1})
		\nonumber\\
		&
		\qquad\qquad\cdot 2\max_{\ell}\sqrt{\max_a w(x_{\ell+1}|a)/w(x_\ell)} \cdot \sqrt{\max_a \mathcal{R}_{ \mathrm{S\text{-}JEPA}}(f,g,a)}
		\nonumber\\
		&\leq 2T \cdot \max_{a,x,x'}\sqrt{w(x_g|a)/w(x)} \cdot \sqrt{\max_a \mathcal{R}_{ \mathrm{S\text{-}JEPA}}(f,g,a)}.
	\end{align}
\end{proof}

\begin{proof}[Proof of Theorem \ref{thm::approx}]
	
	By Theorem \ref{thm::spectral_jepa}, given $a \in \mathcal{A}$, we have
	
	\begin{align}
		\mathcal R_{\mathrm{JEPA}}(f,g,a)
		=
		\|\bar M(a)-G(F,a)^\top F\|_F^2
		+\mathrm{const}.
	\end{align}
	Since $B:=G(F,a)^\top F$ has rank at most $k$ minimizing the JEPA population risk is equivalent to solving
	\begin{align}
		\min_{\operatorname{rank}(B)\le k}
		\|\bar M(a)-B\|_F^2.
	\end{align}
	Let
	\begin{align}
		\bar{M}(a)
		= U(a)\Sigma(a)V(a)^\top
	\end{align}
	be the singular value decomposition of $\bar{M}(a)$.
	By the Eckart--Young--Mirsky theorem, the unique best rank-$k$ approximation of $\bar{M}(a)$ under the Frobenius norm is
	\begin{align}
		B^* := G^*(F,a)^{\top}F^*
		=U_k(a)\Sigma_k(a)V_k(a)^\top.
	\end{align}
	Then one valid factorization is obtained by distributing the singular values equally between the two factors, i.e.,
	\begin{align}
		F^*
		=\Sigma_k(a)^{1/2}V_k(a)^\top,
	\end{align}
	and
	\begin{align}
		G^*(F,a)
		=\Sigma_k(a)^{1/2}U_k(a)^\top.
	\end{align}
	In this case, we have 
	\begin{align}
		\|\bar{M}(a)-G^*(F,a)^{\top}F^*\|_F^2
		=\sum_{i>k}\sigma_i^2(a),
	\end{align}
	where $\sigma_i(a)$ is the $i$-th largest singular value of $\bar{M}(a)$ for $i \in \{1, \ldots, \#|\mathcal{X}|\}$.
	Hence we have
	\begin{align}
		\mathcal R_{ \mathrm{S\text{-}JEPA}}(f^*,g^*,a) = \min_{f,g} \mathcal R_{ \mathrm{S\text{-}JEPA}}(f,g,a)
		=\sum_{i>k}\sigma_i^2(a).
	\end{align}
\end{proof}

\begin{proof}[Proof of Theorem \ref{thm::sample_error}]
	Fix $a\in\mathcal A$. Define the composed hypothesis class
	\[
	\mathcal H_a
	:=
	\left\{
	h_{f,g,a}:=g\circ f\circ a:
	f\in\mathcal F,\ g\in\mathcal G
	\right\},
	\]
	where $g\circ f\circ a(z,z^+,z'):= \big(g(f(z_i),a)^\top f(z_i')\big)^2 - 2g(f(z_i),a)^\top f(z_i^+)$. 
	For simplicity, write $h(z) := g(f(z), a)$. The population spectral JEPA risk can be written as
	\[
	\mathcal R_{\mathrm{S\text{-}JEPA}}(f,g,a)
	=
	\mathbb E_{(z,z^+,z^-)}
	\left[
	\left(h(z)^\top f(z')\right)^2
	-
	2h(z)^\top f(z^+)
	\right],
	\]
	where $(z,z^+)$ is a positive pair and $z'$ is a negative sample. Its empirical
	counterpart based on $m=n/2$ independent tuples is
	\[
	\hat{\mathcal R}_{m}(f,g,a)
	=
	\frac1m
	\sum_{i=1}^m
	\left[
	\left(h(z_i)^\top f(z'_i)\right)^2
	-
	2h(z_i)^\top f(z_i^+)
	\right].
	\]
	
	Since $\|f\|_\infty\le \kappa$ and $\|g\|_\infty\le \kappa$, we have
	\[
	\|h(z)\|_\infty\le \kappa,
	\qquad
	\|f(z)\|_\infty\le \kappa .
	\]
	Therefore,
	\[
	\left|h(z)^\top f(z')\right|
	\le
	\sum_{\ell=1}^k |h_\ell(z)|\,|f_\ell(z')|
	\le
	k\kappa^2.
	\]
	Consequently, the tuple loss
	\[
	\ell_{f,g,a}(z,z^+,z^-)
	:=
	\left(h(z)^\top f(z^-)\right)^2
	-
	2h(z)^\top f(z^+)
	\]
	is uniformly bounded as
	\[
	\ell_{f,g,a}(z,z^+,z^-)
	\in
	\left[
	-2k\kappa^2,\;
	k^2\kappa^4+2k\kappa^2
	\right].
	\]
	We now control the Rademacher complexity of the tuple-loss class
	\[
	\mathcal L_a
	:=
	\left\{
	\ell_{f,g,a}:
	f\in\mathcal F,\ g\in\mathcal G
	\right\}.
	\]

	We now give the detailed Rademacher complexity bound for the tuple-loss
	class. Recall that
	\[
	\ell_{f,g,a}(z,z^+,z^-)
	:=
	\left(h(z)^\top f(z^-)\right)^2
	-
	2h(z)^\top f(z^+),
	\qquad
	h:=g\circ f\circ a .
	\]
	Let
	\[
	\mathcal H_a
	:=
	\{h_{f,g,a}=g\circ f\circ a:\ f\in\mathcal F,\ g\in\mathcal G\}.
	\]
	
	Since $a$ is fixed, for any sample $S=\{(z_i,z_i^+,z_i^-)\}_{i=1}^m$, we have
	\[
	\hat{\mathfrak R}_{S}(\mathcal H_a)
	=
	\hat{\mathfrak R}_{a(S)}(\mathcal G\circ\mathcal F),
	\]
	where $a(S):=\{(a,z_i)\}_{i=1}^m$.
	If $\hat{\mathfrak R}_{m}(\mathcal G\circ\mathcal F)$ denotes the worst-case
	empirical Rademacher complexity over all samples of size $m$, then
	\[
	\hat{\mathfrak R}_{S}(\mathcal H_a)
	\le
	\hat{\mathfrak R}_{m}(\mathcal G\circ\mathcal F).
	\]
	
	Let $S=\{(z_i,z_i^+,z_i^-)\}_{i=1}^m$ be a fixed sample of $m$ tuples.
	We need to bound
	\[
	\hat{\mathfrak R}_{S}(\mathcal L_a)
	:=
	\mathbb E_{\sigma}
	\left[
	\sup_{f\in\mathcal F,\;g\in\mathcal G}
	\frac1m
	\sum_{i=1}^m
	\sigma_i
	\ell_{f,g,a}(z_i,z_i^+,z_i^-)
	\right],
	\]
	where $\sigma_1,\ldots,\sigma_m$ are i.i.d. Rademacher random variables.
	By subadditivity of the supremum,
	\begin{align}
		\hat{\mathfrak R}_{S}(\mathcal L_a)
		\le
		\underbrace{
			\mathbb E_{\sigma}
			\left[
			\sup_{f,g}
			\frac1m
			\sum_{i=1}^m
			\sigma_i
			\left(h(z_i)^\top f(z_i^-)\right)^2
			\right]
		}_{T_1}
		+
		2\underbrace{
			\mathbb E_{\sigma}
			\left[
			\sup_{f,g}
			\frac1m
			\sum_{i=1}^m
			\sigma_i
			h(z_i)^\top f(z_i^+)
			\right]
		}_{T_2}.
	\end{align}
	
	We first control $T_1$. Since
	\[
	|h(z_i)^\top f(z_i^-)|\le k\kappa^2,
	\]
	the map $u\mapsto u^2$ is $2k\kappa$-Lipschitz on
	$[-k\kappa^2,k\kappa^2]$. Therefore, by Talagrand's contraction lemma,
	\begin{align}
		T_1
		&\le
		2k\kappa^2
		\,
		\mathbb E_{\sigma}
		\left[
		\sup_{f,g}
		\frac1m
		\sum_{i=1}^m
		\sigma_i
		h(z_i)^\top f(z_i^-)
		\right].
	\end{align}
	Expanding the inner product gives
	\begin{align}
		T_1
		&\le
		2k\kappa^2
		\sum_{\ell=1}^k
		\mathbb E_{\sigma}
		\left[
		\sup_{h_\ell\in\mathcal H_{a,\ell},\;f_\ell\in\mathcal F_\ell}
		\frac1m
		\sum_{i=1}^m
		\sigma_i
		h_\ell(z_i)f_\ell(z_i^-)
		\right].
	\end{align}
	For two scalar function classes $\mathcal U$ and $\mathcal V$ uniformly
	bounded by $\kappa$, the following standard product-class bound holds:
	\begin{align}
		\mathbb E_{\sigma}
		\left[
		\sup_{u\in\mathcal U,\;v\in\mathcal V}
		\frac1m
		\sum_{i=1}^m
		\sigma_i u(x_i)v(y_i)
		\right]
		\le
		4\kappa
		\left[
		\hat{\mathfrak R}_{m}(\mathcal U)
		+
		\hat{\mathfrak R}_{m}(\mathcal V)
		\right].
		\label{eq:product-class-bound}
	\end{align}
	Indeed, using the identity
	\[
	uv
	=
	\frac12\left[(u+v)^2-u^2-v^2\right],
	\]
	or equivalently the polarization identity, and then applying Talagrand's
	contraction lemma to the square function on $[-2\kappa,2\kappa]$, one obtains
	\eqref{eq:product-class-bound}. Applying this bound with
	\[
	\mathcal U=\mathcal H_{a,\ell},
	\qquad
	\mathcal V=\mathcal F_\ell,
	\]
	we get
	\begin{align}\label{eq:T1-bound}
		T_1
		\le
		2k\kappa^2
		\sum_{\ell=1}^k
		4\kappa
		\left[
		\hat{\mathfrak R}_{m}(\mathcal H_{a,\ell})
		+
		\hat{\mathfrak R}_{m}(\mathcal F_\ell)
		\right]
		\le
		8k^2\kappa^3
		\left[
		\hat{\mathfrak R}_{m}(\mathcal H_a)
		+
		\hat{\mathfrak R}_{m}(\mathcal F)
		\right].
	\end{align}

	We next control $T_2$. Expanding the inner product gives
	\begin{align}
		T_2
		&=
		\mathbb E_{\sigma}
		\left[
		\sup_{f,g}
		\frac1m
		\sum_{i=1}^m
		\sigma_i
		\sum_{\ell=1}^k h_\ell(z_i)f_\ell(z_i^+)
		\right]
		\nonumber\\
		&\le
		\sum_{\ell=1}^k
		\mathbb E_{\sigma}
		\left[
		\sup_{h_\ell\in\mathcal H_{a,\ell},\;f_\ell\in\mathcal F_\ell}
		\frac1m
		\sum_{i=1}^m
		\sigma_i
		h_\ell(z_i)f_\ell(z_i^+)
		\right].
	\end{align}
	Again applying the product-class bound \eqref{eq:product-class-bound},
	\begin{align}
		T_2
		\le
		\sum_{\ell=1}^k
		4\kappa
		\left[
		\hat{\mathfrak R}_{m}(\mathcal H_{a,\ell})
		+
		\hat{\mathfrak R}_{m}(\mathcal F_\ell)
		\right]
		\le
		4k\kappa
		\left[
		\hat{\mathfrak R}_{m}(\mathcal H_a)
		+
		\hat{\mathfrak R}_{m}(\mathcal F)
		\right].
		\label{eq:T2-bound}
	\end{align}
	
	Combining \eqref{eq:T1-bound} and \eqref{eq:T2-bound}, we obtain
	\begin{align}
		\hat{\mathfrak R}_{S}(\mathcal L_a)
		&\le
		T_1+2T_2
		\nonumber\\
		&\le
		\left(8k^2\kappa^3+8k\kappa\right)
		\left[
		\hat{\mathfrak R}_{m}(\mathcal H_a)
		+
		\hat{\mathfrak R}_{m}(\mathcal F)
		\right]
		\nonumber\\
		&\le
		\left(8k^2\kappa^3+8k\kappa\right)
		\left[
		\hat{\mathfrak R}_{m}(\mathcal G\circ\mathcal F)
		+
		\hat{\mathfrak R}_{m}(\mathcal F)
		\right].
	\end{align}
	Taking $m=n/2$ gives the desired Rademacher complexity control for the
	tuple-loss class.
	
	Here $\hat{\mathfrak R}_{m}(\mathcal G\circ\mathcal F)$ controls the online
	branch $h=g\circ f\circ a$, while $\hat{\mathfrak R}_{m}(\mathcal F)$ controls
	the target branch $f$. Since $a$ is fixed, the composition with $a$ does not
	introduce an additional trainable function class and is absorbed into
	$\mathcal G\circ\mathcal F$.
	
	By the standard Rademacher uniform convergence bound, with probability at
	least $1-\delta$, uniformly over $f\in\mathcal F$ and $g\in\mathcal G$,
	\begin{align}
		&
		\left|
		\mathcal R_{\mathrm{S\text{-}JEPA}}(f,g,a)
		-
		\hat{\mathcal R}_{m}(f,g,a)
		\right|
		\nonumber\\
		&\le
		\left(16k^2\kappa^3+16k\kappa\right)
		\left[
		\hat{\mathfrak R}_{m}(\mathcal G\circ\mathcal F)
		+
		\hat{\mathfrak R}_{m}(\mathcal F)
		\right]
		+
		\left(4k\kappa^2+k^2\kappa^4\right)
		\left(
		\sqrt{\frac{\log(2/\delta)}{n}}
		+
		\delta
		\right).
	\end{align}
	Taking $m=n/2$, define
	\begin{align}
		\epsilon_n
		:=
		\left(16k^2\kappa^3+16k\kappa\right)
		\left[
		\hat{\mathfrak R}_{n/2}(\mathcal G\circ\mathcal F)
		+
		\hat{\mathfrak R}_{n/2}(\mathcal F)
		\right]
		+
		\left(4k\kappa^2+k^2\kappa^4\right)
		\left(
		\sqrt{\frac{\log(2/\delta)}{n}}
		+
		\delta
		\right).
	\end{align}
	Then, with probability at least $1-\delta$, uniformly over
	$f\in\mathcal F$ and $g\in\mathcal G$,
	\[
	\left|
	\mathcal R_{\mathrm{S\text{-}JEPA}}(f,g,a)
	-
	\hat{\mathcal R}_{n/2}(f,g,a)
	\right|
	\le
	\epsilon_n .
	\]
	
	Using this uniform deviation bound and the empirical optimality of
	$(\hat f,\hat g)$, we have
	\begin{align}
		&\mathcal R_{\mathrm{S\text{-}JEPA}}(\hat f,\hat g,a)
		-
		\mathcal R_{\mathrm{S\text{-}JEPA}}(f^*,g^*,a)
		\nonumber\\
		&\le
		\hat{\mathcal R}_{n/2}(\hat f,\hat g,a)
		+
		\epsilon_n
		-
		\mathcal R_{\mathrm{S\text{-}JEPA}}(f^*,g^*,a)
		\nonumber\\
		&\le
		\hat{\mathcal R}_{n/2}(f^*,g^*,a)
		+
		\epsilon_n
		-
		\mathcal R_{\mathrm{S\text{-}JEPA}}(f^*,g^*,a)
		\nonumber\\
		&\le
		2\epsilon_n .
	\end{align}
	Substituting the definition of $\epsilon_n$ gives
	\begin{align}
		&\mathcal R_{\mathrm{S\text{-}JEPA}}(\hat f,\hat g,a)
		-
		\mathcal R_{\mathrm{S\text{-}JEPA}}(f^*,g^*,a)
		\nonumber\\
		&\le
		\left(32k^2\kappa^3+32k\kappa\right)
		\left[
		\hat{\mathfrak R}_{n/2}(\mathcal G\circ\mathcal F)
		+
		\hat{\mathfrak R}_{n/2}(\mathcal F)
		\right]
		\nonumber\\
		&\quad+
		\left(8k\kappa^2+2k^2\kappa^4\right)
		\left(
		\sqrt{\frac{\log(2/\delta)}{n}}
		+
		\delta
		\right).
	\end{align}
	This proves the desired sample error bound.
\end{proof}

\subsection{Effect of the Additional Uniformity Regularization}
\label{app:uniformity}

Our theoretical formulation introduces an additional uniformity regularization term to the original JEPA objective, which enables the spectral characterization used in our analysis. Since this term is introduced primarily for theoretical analysis and is not part of the original JEPA objective, a natural question is whether it changes the empirical performance of the underlying world model.

To examine this, we conduct an additional experiment based on the representative JEPAL-based world model LeWM \citep{maes2026leworldmodel} and follow its original experimental setting, including the model architecture, training protocol, and downstream planning procedure. The original LeWM objective consists of the latent prediction MSE loss together with SIGReg. We additionally incorporate the uniformity term used in our theoretical formulation:
\begin{equation}
    \mathcal{L}
    =
    \mathcal{L}_{\mathrm{MSE}}
    +
    \mathcal{L}_{\mathrm{SIGReg}}
    +
    \mathcal{L}_{\mathrm{uni}},
\end{equation}

\begin{table}[h]
\centering
\caption{Effect of the additional uniformity regularization used in our theoretical formulation. We follow the original LeWM setting and only add the uniformity term. Each result is the mean value evaluated over three seeds with 50 episodes per seed.}
\vspace{1mm}
\label{tab:uniformity}
\begin{tabular}{lccc}
\toprule
Objective & PushT & Cube & TwoRoom \\
\midrule
LeWM & 90.7 & 66.7 & 84.7 \\
LeWM $+\mathcal{L}_{\mathrm{uni}}$
& \textbf{91.3} & \textbf{69.3} & \textbf{85.3} \\
\bottomrule
\end{tabular}
\end{table}

We evaluate the resulting models on PushT, Cube, and TwoRoom following the original LeWM evaluation protocol with a planning horizon of $H=25$. Each result is evaluated over three seeds with 50 episodes per seed, resulting in 150 episodes in total. As shown in Table~\ref{tab:uniformity}, adding the uniformity term does not degrade performance on any of the three environments. Instead, we observe small improvements from $90.7\%$ to $91.3\%$ on PushT, $66.7\%$ to $69.3\%$ on Cube, and $84.7\%$ to $85.3\%$ on TwoRoom. These results suggest that the additional uniformity regularization introduced in our theoretical formulation is empirically benign and can even provide modest performance improvements.

\subsection{Details of the Synthetic Control Environment}
\label{app:synthetic}

We provide additional implementation details for the controlled experiment in Section~\ref{sec:latent-input}.

\paragraph{Environment.}
We use a 2D point-mass system with underlying state
$s_t=[p_x,p_y,v_x,v_y]$ and action $a_t=[a_x,a_y]$.
The dynamics are given by
\begin{equation}
    v_{t+1}
    =
    \mathrm{damping}\cdot v_t
    +
    \mathrm{action\_scale}\cdot a_t,
    \qquad
    p_{t+1}
    =
    p_t + dt\cdot v_{t+1}.
\end{equation}
The model does not directly observe the underlying state. Instead, its
input is a higher-dimensional observation
\begin{equation}
    x_t=
    [x_t^{\mathrm{state}},
     x_t^{\mathrm{nuisance}},
     x_t^{\mathrm{noise}}],
\end{equation}
where $x_t^{\mathrm{state}}$ contains features derived from the position
and velocity, $x_t^{\mathrm{nuisance}}$ contains additional nonlinear
projections of the state, and $x_t^{\mathrm{noise}}$ consists of random
features independent of both the action and task objective.

\paragraph{Predictive models.}
Both latent- and input-level models use the same predictive backbone,
consisting of an observation encoder $x_t\mapsto z_t$, an action
encoder, and a GRU-based dynamics predictor. An additional decoder
$z\mapsto x$ is used for input-level reconstruction.

Given the current observation $x_t$ and an action sequence
$a_{t:t+H-1}$, the model predicts the state after $H$ steps. We use
$H=3$ throughout the experiment. The latent-level model is trained to
predict the encoded future observation,
\begin{equation}
    \mathcal{L}_{\mathrm{latent}}
    =
    \mathrm{MSE}
    \left(
        \hat z_{t+H},
        \mathrm{stopgrad}(\mathrm{Enc}(x_{t+H}))
    \right),
\end{equation}
together with a lightweight variance regularization term. The
input-level model decodes the predicted latent and minimizes the
reconstruction error
\begin{equation}
    \mathcal{L}_{\mathrm{input}}
    =
    \mathrm{MSE}
    \left(
        \hat x_{t+H},
        x_{t+H}
    \right).
\end{equation}

To control the influence of unpredictable observation components, we
vary the reconstruction weight assigned to the random-noise dimensions.
This quantity corresponds to the \emph{noise level} reported in
Figure~\ref{fig::synthetic_latent_input}. Increasing this weight places
greater emphasis on reconstructing information that is inherently
unpredictable from the current state and action.

\paragraph{Planning.}
For each evaluation episode, we sample a start state and a goal state.
The corresponding goal observation $x_{\mathrm{goal}}$ is provided to
the planner, and in the latent setting it is encoded as
$z_{\mathrm{goal}}=\mathrm{Enc}(x_{\mathrm{goal}})$.

We use receding-horizon planning with the Cross-Entropy Method (CEM).
At each decision step, CEM optimizes a candidate action sequence using
the learned predictive model, the first action is executed, and the
agent replans from the resulting observation. The latent-level planner
minimizes the distance between the predicted latent and
$z_{\mathrm{goal}}$, whereas the input-level planner minimizes the
corresponding distance in observation space.

The \emph{required planning steps} in
Figure~\ref{fig::synthetic_latent_input} controls the approximate
number of such closed-loop decisions needed to reach the goal. It is
therefore distinct from the fixed training prediction horizon $H=3$.
Larger values correspond to tasks requiring more repeated model
predictions and replanning steps.

All methods are evaluated using the same ground-truth positional
criterion. A rollout is counted as successful when
\begin{equation}
    \|p_{\mathrm{current}}-p_{\mathrm{goal}}\|_2 \leq 0.08.
\end{equation}

\end{document}